\documentclass[twoside]{article}

\usepackage{aistats2020}
\usepackage[utf8]{inputenc} 
\usepackage[T1]{fontenc}    
\usepackage{hyperref}       
\usepackage{url}            
\usepackage{booktabs}       
\usepackage{amsfonts}       
\usepackage{nicefrac}       
\usepackage{microtype}      
\usepackage{hyperref}
\usepackage{url}
\usepackage{amssymb} 
\usepackage{amsthm} 
\usepackage{epstopdf} 
\usepackage{float}
\usepackage{mathtools}

\usepackage{cancel}
\usepackage{siunitx}
\usepackage{array}
\usepackage{multirow}
\usepackage{amssymb}
\usepackage{gensymb}
\usepackage{tabularx}
\usepackage{booktabs}
\usepackage{mathdots}
\usepackage{tabularx}
\usepackage{bm}
\usepackage{esvect}
\usepackage{algorithm2e}
\allowdisplaybreaks
\usepackage[symbol]{footmisc}

\newcommand{\removelatexerror}{\let\@latex@error\@gobble}

\newcommand{\R}{\mathbb{R}}
\newcommand{\cA}{\mathcal{A}}
\newcommand{\cB}{\mathcal{B}}

\newcommand{\cN}{\mathcal{N}}
\newcommand{\cO}{\mathcal{O}}

\newcommand{\cS}{\mathcal{S}}

\newcommand{\cU}{\mathcal{U}}
\newcommand{\cV}{\mathcal{V}}

\newtheorem{theorem}{Theorem}[section]

\newtheorem{lemma}[theorem]{Lemma}
\newcommand{\argmax}{\mathop{\mathrm{argmax}}}

\renewcommand{\vec}[1]{\ensuremath{\bm{#1}}}

\newcommand{\mat}[1]{\ensuremath{\mathbf{#1}}}

\newcommand{\ten}[1]{\mat{\ensuremath{\boldsymbol{\mathcal{#1}}}}}

\DeclareMathOperator*{\Prob}{\mathbb{P}}

\newcommand{\longmethod}{unnormalized Q function}
\newcommand{\longmethodcaps}{Unnormalized Q function}
\newcommand{\method}{UQF}

\begin{document}

\twocolumn[

\aistatstitle{Efficient Planning under Partial Observability with\\ Unnormalized Q Functions and Spectral Learning}
\vspace{-0.5\baselineskip}
\aistatsauthor{ Tianyu Li\footnotemark[1]\footnotemark[2]\footnotemark[3] \And Bogdan Mazoure\footnotemark[1]\footnotemark[2]\footnotemark[3]\And Doina Precup\footnotemark[2]\footnotemark[3]\footnotemark[5]\footnotemark[6]\And  Guillaume Rabusseau\footnotemark[3]\footnotemark[4]\footnotemark[5]}\vspace{0.5\baselineskip}]

\begin{abstract}
Learning and planning in partially-observable domains is one of the most difficult problems in reinforcement learning. Traditional methods consider these two problems as independent, resulting in a classical two-stage paradigm: first learn the environment dynamics and then plan accordingly. This approach, however, disconnects the two problems and can consequently lead to algorithms that are sample inefficient and time consuming. In this paper, we propose a novel algorithm that combines learning and planning together. Our algorithm is closely related to the spectral learning algorithm for predicitive state representations and offers appealing theoretical guarantees and time complexity. We empirically show on two domains that our approach is more sample  and time efficient compared to  classical methods.
 
\end{abstract}

 \section{Introduction}

A common assumption in reinforcement learning (RL) is that the agent has the knowledge of the entire dynamics of the environment, including the state space, transition probabilities, and a reward model. However, in many real world applications, this assumption may not always be valid. Instead, the environment is often \emph{partially observable}, meaning that the true state of the system is not completely visible to the agent. This partial observability can result in numerous difficulties in terms of learning the dynamics of the environment and planning to maximize returns. 

Partially observable Markov decision Processes
(POMDPs)~\cite{sondik1978optimal,cassandra1994acting} provide a formal framework for single-agent planning under a partially observable environment.  In contrast with MDPs, agents in POMDPs do not have direct access to the state space. Instead of observing the states, agents only have access to observations and need to operate on the so-called \emph{belief states}, which describe the distribution over the state space given some past trajectory. Therefore, POMDPs model the dynamics of an RL environment in a latent variable fashion and explicitly reason about uncertainty in both action effects
and state observability~\cite{boots2011closing}. Planning under a POMDP has long been considered a difficult problem~\cite{kaelbling1998planning}. To perform \emph{exact} planning under a POMDP, one common approach is to optimize the value function
over all possible belief states. Value iteration for POMDPs~\cite{sondik1978optimal} is one particular example of this approach. However, due to the curse of dimensionality and curse of history~\cite{pineau2006anytime}, this   method is often computationally intractable for most realistic POMDP planning problems~\cite{boots2011closing}. 

As an alternative to exact planning, the family of predictive state representations~(PSRs) has attracted many interests. In fact, PSRs are no weaker than POMDPs in terms of their representation power~\cite{littman2002predictive}, and there are many efficient algorithms to estimate PSRs and their variants relying on likelihood based algorithms~\cite{singh2003learning,singh2004predictive} or spectral learning techniques~\cite{boots2011closing, hamilton2013modelling}. However, to plan with PSRs is not straight-forward. Typically, a two-stage process is applied to discover the optimal policy with PSRs: first, a PSR model is learned in an unsupervised fashion, then a planning method is used to discover the optimal policy based on the learned dynamics. Several planning algorithms can be used for the second stage of this process. For example in~\cite{boots2011closing,izadi2008point}, a reward function is estimated with the learned PSRs, and then combined with point based value iteration (PBVI)~\cite{pineau2003point} to obtain an approximation of the optimal policy; in~\cite{hamilton2014efficient}, the authors use the fitted-Q method~\cite{ernst2005tree} to iteratively regress Bellman updates on the learned state representations, thus approximating the action value function.  
\footnotetext{\footnotemark[1] Equal Contribution \footnotemark[3] Quebec AI Institute (MILA) \footnotemark[2] McGill University \footnotemark[4] Université de Montréal \footnotemark[5] Canadian Institute for Advanced Research (CIFAR) Chair \footnotemark[6]Deepmind Montreal}

However, despite numerous successes, this two-stage process still suffers from  significant drawbacks. To begin with, the PSRs parameters are learned independently from the reward information, resulting in a less efficient representation for planning. Secondly, planning with PSRs often involves multiple stages of regression, and these extra steps of approximation can be detrimental for obtaining the optimal policy. Last but not least, the planning methods for PSRs are often iterative methods that can be very time consuming. 

In this work, we propose an alternative to the traditional paradigm of planning in partially observable environments. Inspired by PSRs, our method leverages the spectral learning algorithm for subspace identification, treating the environment as a latent variable model. However, instead of explicitly learning the dynamics of the environment, we learn a function that is proportional to the action value function, which we call \emph{unnormalized Q function} (UQF). In doing so, we incorporate the reward information into the dynamics in a supervised learning fashion., which unifies the two stages of the classical learning-planning paradigm for POMDPs. In some sense, our approach effectively learns a goal-oriented representation of the environment. Therefore, in terms of planning, our method is more sample efficient compared to the two-stage learning paradigm (for example, PSRs).
Our algorithm relies on the spectral learning algorithm for \emph{weighted finite automata}~(WFAs), which are an extension of PSRs that can model not only probability distributions, but arbitrary real-valued functions. Our method inherits  the benefits of spectral learning: it provides a consistent estimation of the UQF and is computationally more efficient than EM based methods. 
Furthermore, planning with PSR usually requires multiple steps and often uses iterative planning method, which can be time consuming. In contrast, our algorithm directly learns a policy in one step, offering a more time efficient method. In addition, we also adopt \emph{matrix compressed sensing} techniques to extend this approach to complex domains. This technique has also been used in PSRs based methods to overcome similar problems~\cite{hamilton2014efficient}.

We conduct experiments on partially observable grid world and S-PocMan environment~\cite{hamilton2014efficient} where we compare our approach with classical PSR based methods. In both domains, our approach is significantly more data-efficient than PSR based methods with considerably smaller running time.

\section{Background}
 In this section, we will introduce some basic RL concepts, including partially observable Markov decision processes (POMDPs), predictive state representations (PSRs) and their variants as well as the notion of WFAs. We will also introduce the spectral learning algorithms for WFAs. 
 
\subsection{Partially Observable Markov Decision Processes (POMDPs)}
Markov decision processes have been widely applied in the field of reinforcement learning. A Markov decision process (MDP) of size $k$ is characterized by a 6-tuple $\langle \ten{T}, \mat{r}, \cA, \cS,\vec{\mu},\gamma\rangle$ where $\ten{T} \in [0, 1]^{\cS\times \cA \times \cS}$  is the transition probability; $\bm{\mu} \in [0,1]^{\cS}$ is the initial state distribution; $\mat{r}\in \mathbb{R}^{\cS}$ is the reward vector over states; $\gamma\in [0,1)$ is the discount factor; $\cS$ is the set of states and $\cS = \{s^1, \cdots s^k\}$ and $\cA$ is the set of actions.
The goal of an RL task is often to learn a policy that governs the actions of the agent to maximize the \emph{accumulated discounted rewards} (return) in the future. A policy in an MDP environment is defined as $\mat{\Pi}\in [0,1]^{\cS\times \cA}$. $\mat{\Pi}$ operates at the state level. At each timestep, the optimal action is selected probabilistically with respect to $\mat{\Pi}$ given the state of the current step. The agent then move to the next state depending on the corresponding transition matrix indexed by $a$ and collect potential rewards from the state. 

 However, in practice, it is rarely the case that we can observe the exact state of the agent. For example, in the game of poker, the player only knows the cards at hands and this information alone does not determine the exact state of the game. Partially observable Markov decision processes (POMDPs) were introduced to model this type of problems. Under the POMDP setting, the true state space of the model is hidden from the agent through partial observability. That is, an observation $o_t$ is obtained probabilistically based on the agent's current state and the observation emission probability. A \textit{partially observable Markov decision process} (POMDP) is characterized by an 8-tuple $\langle \ten{T},\ten{O}, \mat{r}, \cA,\cO, S,\vec{\mu},\gamma\rangle$, where, $\cO$ is a set of observations and $\ten{O}\in \mathbb{R}^{S\times A\times O}$ is the observation emission probability  and the rest parameters follow the definitions in MDPs. 
 
 As the agent cannot directly observe which state it is at, one classic problem in POMDP is to compute the \emph{belief state} $\bm{b}(h)\in\mathbb{R}^\cS$ knowing the past trajectory $h$. Formally, given $h = a_1o_1\cdots a_no_n \in (\cA\times \cO)^*$, we want to compute $\bm{b}(h)^\top = [\mathbb{P}(s^1|h), \cdots, \mathbb{P}(s^k|h)]^\top$. This can be solved with a forward method similar to HMM~\cite{juang1991hidden}. Let $\Tilde{\mat{O}}_{ao} = \mathrm{diag}(\ten{O}_{s^1, a, o}, \ten{O}_{s^2, a, o}, \cdots, \ten{O}_{s^k, a, o})$,  $\Tilde{\mat{M}}_a = \mathrm{diag}(\bm{\Pi}_{s^1, a},\bm{\Pi}_{s^2, a}, \cdots, \bm{\Pi}_{s^k, a})$ and denote $\mat{E}_{ao} = \Tilde{\mat{M}}_{a}\ten{T}_{:, a, :}\Tilde{\mat{O}}_{ao}$. It can be shown that $\bm{b}(h)^\top = \bm{\mu}^\top \mat{E}_{a_1o_1}\cdots \mat{E}_{a_no_n}$ and $\bm{b}(\lambda) = \bm{\mu}^\top$, where $\lambda$ denotes the empty string. 
 
Similarly to the MDP setting, the state-level policy for POMDPs is  defined by $\mat{\Pi}\in [0,1]^{\mathcal{S}\times \mathcal{A}}$, where $\mat{\Pi}_{s,a} = \mathbb{P}(a|s)$. However, due to the partial observability, the agent's true state cannot be directly observed. 
Nonetheless, any state-level policy implicitly induces a probabilistic policy over past trajectories, defined by ${\Pi}(a|h) = \sum_{s\in\cS}\mathbb{P}(s|h)\mat{\Pi}_{s,a}$ for each $h \in (\cA\times \cO)^*$. Similarly, every state-level policy induces a probabilistic distribution over trajectories. With a slight abuse of notation, denote the probability of a trajectory $h$ under the policy $\Pi$ by $\mathbb{P}^\Pi(h)$. Here, we assume $\Pi$ is induced by a state-level policy $\mat{\Pi}$ and define 
$\mathbb{P}^\Pi(h) =  \vec{b}(h)^\top \vec{1}$,
where $\vec{1}$ is an all-one vector. To make clear of the notations, we will use $\pi: \Sigma^* \to \cA$ for deterministic policies in the later sections. 
%
\subsection{Predictive state representations}
One common approach for modelling the dynamics of a POMDP is the so-called predictive state representations (PSRs). A PSR is a model of a dynamical system in which the current state is represented as
a set of predictions about the future behavior of the system~\cite{littman2002predictive, singh2004predictive}. This is done by maintaining a set of action-observation pairs, called \emph{tests}, and the representation of the current state is given by the conditional probabilities of these tests given the past trajectory, which is referred to as \emph{history}.  Although there are multiple methods to select a set of tests~\cite{singh2003learning, james2004learning}, it has been shown that with a large action-observation set, finding these tests can be exponentially difficult~\cite{boots2011closing}. 

Instead of explicitly finding the tests,
\emph{transformed predictive state representations} (TPSRs)~\cite{rosencrantz2004learning, boots2011closing} offer an alternative. TPSRs implicitly estimate a linear transformation of the PSR via subspace methods. This approach drastically reduces the complexity of estimating a PSR model and has shown many benefits in different RL domains~\cite{boots2011closing,singh2004predictive}. 

Although this approach is able to obtain a small transformed space of the original PSRs, it still faces scalability problems.  Typically, one can obtain an estimate of TPSR by performing truncated SVD on the estimated \emph{system-dynamics matrix}~\cite{singh2004predictive}, which is indexed by histories and tests. The scalability issue arises in complex domains, which require a large number of histories and tests to form the system-dynamics matrix. As the time complexity of SVD is cubic in the number of histories and tests, the computation time explodes in these types of environments.

Compressed predictive state representations (CPSRs)~\cite{hamilton2013modelling} were introduced to circumvent this issue. The main idea of this approach is to project the high dimensional system-dynamics matrix onto a much smaller subspace spanned by randomly generated bases that satisfy the Johnson-Lindenstrauss (JL) lemma~\cite{johnson1984extensions}. The projection matrices corresponding to these bases are referred to as JL matrices. Intuitively, JL matrices define a low-dimensional embeddings which approximately preserves Euclidean distance between the projected points. More formally, given a matrix $\mat{H}\in \mathbb{R}^{m\times n}$ and JL random projection matrices   $\bm{\Phi}_1\in \mathbb{R}^{m\times d_1}$ and $\bm{\Phi}_2\in \mathbb{R}^{n\times d_2}$, the compressed matrix $\mat{H}_c$ is computed by:
$$\mat{H}_c = \bm{\Phi_1}^\top\mat{H}\bm{\Phi_2}$$
where $\mat{H}_c$ is the compressed matrix. The choice of random projection matrix is rather empirical and often depends on the task. 
Gaussian  matrices~\cite{baraniuk2009random} and Rademacher  matrices~\cite{achlioptas2003database} are common choices for the random projection matrices that satisfy JL lemma. Although does not satisfy JL lemma, hashed random projection have also been shown to preserve certain
kernel-functions and perform extremely well in practice~\cite{weinberger2009feature,shi2009hash}.
\subsection{Weighted finite automata (WFAs)}
In fact, TPSRs (CPSRs) are a subclass of a wider family called weighted finite automata (WFAs). More precisely, TPSRs belong to \emph{stochastic weighted finite automata} (SWFAs)~\cite{droste2009handbook} in the formal language community or \emph{observable operator models} (OOMs)~\cite{jaeger2000observable} in control theory. Further connections between SWFAs, OOMs and TPSRs are shown in~\cite{thon2015links}. WFAs are an extension to TPSRs in the sense that, instead of only computing the probabilities of trajectories, WFAs can compute functions with arbitrary scalar outputs over the given trajectories. Formally, a \emph{weighted finite automaton} (WFA) with $k$ states is a tuple $A=\langle \vec{\alpha},\{\mat{A}_\sigma\}_{\sigma \in \Sigma},\vec{\omega},\Sigma \rangle$, where $\vec{\alpha},\vec{\omega} \in \mathbb{R}^k$ are the initial and terminal weights; $\mat{A}_\sigma \in \mathbb{R}^{k\times k}$ is the transition matrix associated with symbol $\sigma$ from alphabet $\Sigma$. 
Given a trajectory $x = x_1x_2\cdots x_n \in \Sigma^*$, a WFA $A$ computes a function $f_A: \Sigma^* \rightarrow \mathbb{R}$ defined by:
$$f_A(x) = \alpha^\top \mat{A}_{x_1}\mat{A}_{x_2}\cdots \mat{A}_{x_n}\omega$$
We will denote $\mat{A}_{x_1}\mat{A}_{x_2}\cdots \mat{A}_{x_n}$ by $\mat{A}_x$ in the following sections for simplicity. For a function $f: \Sigma^* \to \mathbb{R}$, the \emph{rank} of $f$ is defined as the minimal number of states of a WFA computing $f$. If $f$ cannot be computed by a WFA, we let $\mathrm{rank}(f)=\infty$. 
In the context of TPSRs, we often let $\Sigma = A \times O$ and $f_A$ computes the probability of the trajectory. 

\subsubsection{Hankel matrix}
The learning algorithm of WFAs, relies on the spectral decomposition of the so-called Hankel matrix. The \emph{Hankel matrix} $\mat{H}_f\in\mathbb{R}^{\Sigma^*\times\Sigma^*}$ associated with a function  $f:\Sigma^*\to\mathbb{R}$
is a bi-infinite matrix with entries $(\mat{H}_f)_{u,v}=f(uv)$ for all words $u,v\in\Sigma^*$.
The spectral learning algorithm for WFA relies on the following fundamental relation between the rank of $f$ and the rank of the Hankel matrix $\mat{H}_f$~\cite{carlyle1971realizations,fliess1974matrices}:
\begin{theorem}
For any $f:\Sigma^*\to\mathbb{R}$, $\mathrm{rank}(f)=\mathrm{rank}(\mat{H}_f)$.
\end{theorem}
In practice, one deals with finite sub-blocks of the Hankel matrix. 
Given a basis $\mathcal{B}=(\cU, \cV)\subset \Sigma^*\times \Sigma^*$, where $\cU$ is a set of \emph{prefixes} and $\cV$ is a set of \emph{suffixes},
denote the corresponding sub-block of the Hankel matrix by $\mat{H}_\mathcal{B}\in\R^{\cU\times \cV}$.
For an arbitrary basis $\mathcal{B}=(\cU, \cV)$, define its \emph{p-closure} by $\mathcal{B}^\prime=(\mathcal{U^\prime}, \cV),]$
where $\mathcal{U^\prime}=\cU\cup \cU\Sigma$.
 It turns out that a Hankel matrix over a p-closed basis can be partitioned
into $|\Sigma| + 1$ blocks of the same size~\cite{balle2014spectral}: 
$$\textbf{H}_{\mathcal{B}^\prime}^\top=[\textbf{H}_\lambda^\top| \textbf{H}_{\sigma_1}^\top|\cdots| \textbf{H}_{\sigma_{|\Sigma|}}^\top]$$, 
where $\lambda$ denotes the empty string and for each $\sigma\in\Sigma\cup\{\lambda\}$ the matrix $\mat{H}_{\sigma}\in\R^{\cU\times \cV}$ is defined by
$(\mat{H}_{\sigma})_{u, v}=f(u\sigma v)$.
We say that a basis $\cB =(\cU, \cV)$ is \emph{complete} for the function $f$ if the sub-block $\mat{H}_{\cB}$ has full rank: $\mathrm{rank}(\mat{H}_{\cB}) =\mathrm{rank}{(\mat{H}_f)}$ and we call $\mat{H}_{\cB}$ a \emph{complete sub-block} of $\mat{H}_f$ and $\mat{H}_{\cB^\prime}$ a prefix-closure of $\mat{H}_{\cB}$. It turns out that one can recover the WFA that realizes $f$ via the prefix-closure of a complete sub-block of $\mat{H}_f$~\cite{balle2014spectral} using the \emph{spectral learning} algorithm of WFAs. 

\subsubsection{Spectral learning of WFAs}
It can be shown that the rank of the Hankel matrix $\mat{H}$ is upper bounded by the rank of $f$~\cite{balle2014methods}. Moreover, given a rank factorization of the Hankel matrix $\mat{H} = \mat{PS}$, it is also true that 
$\mat{H}_{\sigma}=\mat{P}\mat{A}_\sigma\mat{S}$ for each $\sigma\in\Sigma$.
The spectral learning algorithm relies on the non-trivial observation that this construction can be reversed:
given any rank $k$ factorization $\mat{H}_\lambda=\mat{PS}$, the WFA  
$A=\langle \bm{\alpha}_0^{\top}, \bm{\alpha}_\infty, \{\textbf{A}_\sigma\}_{\sigma\in\Sigma}\rangle$ defined by
$$\bm{\alpha}_0^{\top}=\mat{P}_{\lambda,:},\ \ \bm{\alpha}_\infty=\mat{S}_{:,\lambda},\ \text{  and
}\textbf{A}_\sigma=\textbf{P}^+ \textbf{H}_\sigma \textbf{S}^+,$$ 
is a minimal WFA computing $f$~\cite[Lemma 4.1]{balle2014spectral}, where
$\mat{H}_\sigma$ for $\sigma\in\Sigma\cup{\lambda}$ denote the finite matrices defined above for a prefix closed complete
basis $\mathcal{B}$. In practice, we compute empirical estimates of the Hankel matrices such that $\hat{\mat{H}}_{u,v} = \frac{1}{|D|}\sum_{i \in |D|}\mathbb{I}_{uv}(D_i)\vec{y}_i$, where $D$ is a dataset of trajectories, $\vec{y}$ is a vector of the outputs of $D$, $\mathbb{I}_{uv}(D_i) = 1$ if $uv = D_i$ and zero otherwise. 


In fact, the above algorithm can be readily used for learning TPSRs. In PSRs terminology, the prefixes are the histories while the suffixes are the tests and the alphabet $\Sigma$ is the set of all possible action observation pairs, i.e. $\Sigma = \cA\times\cO$. By simply replacing Hankel matrix $\mat{H}$ by the system-dynamics matrix proposed in~\cite{singh2004predictive}, one will exactly fall back to TPSRs learning algorithm~\cite{singh2004predictive, boots2011closing}.

\section{Planning with Unnormalized Q Function}

In this section, we will introduce our POMDP planning method. The main idea of our algorithm is to directly  compute the optimal policy based on the estimation of \emph{\longmethod{}} that is proportional to the action value function. Moreover, the value of this function, given a past trajectory, can be computed via a WFA and it is then straight-forward to use the classical spectral learning algorithm to recover this WFA. Unlike traditional PSR methods, our approach takes advantage of the reward information by integrating the reward into the learned representations. In contrast, classical PSRs based methods construct the representations solely with the environment dynamics, completely ignoring the reward information. Consequently, our method offers a more sample efficient representation of the environment for planning under POMDPs. In addition, our algorithm only needs to construct a WFA and there is no other iterative planning method involved. Therefore, compared to traditional methods to plan with PSRs, our algorithm is more time efficient. Finally, with the help of compressed sensing techniques, we are able to scale our algorithm to complex domains.
\subsection{\longmethodcaps{}}
The estimation of the action value function is of great importance for planning under POMDP.  Typically, given a probabilistic sampling policy $\Pi: \cA\times\Sigma^* \to [0,1]$, where $\Sigma = \cA\times\cO$, the action value function (Q function) of a given trajectory $h \in\Sigma^*$ is defined by:
\begin{align*}
    Q^\Pi(h, a) = \mathbb{E}^\Pi(r_t + \gamma r_{t+1} + \cdots + \gamma^i r_{t+i} +\cdots|ha)
\end{align*}
where $|h| = t$ and $r_t$ is the immediate rewards collected at time step $t$. 

Given a POMDP $\psi = \langle \ten{T},\ten{O}, \mat{r}, \cA,\cO, \cS,\vec{\mu},\gamma\rangle$, denote the expected immediate reward collected after $h$ by $\tilde{R}(h)$, which is defined as:
    $$\tilde{R}(h) = \mathbb{E}_{s\in \cS}(\vec{r}_s|h)
    = \sum_{s \in \cS}\vec{r}_s\mathbb{P}(s|h)$$
The action value function can then be expanded to:
\begin{align*}
    Q^\Pi(h, a) &= \mathbb{E}^\Pi(r_t + \gamma r_{t+1} + \cdots + \gamma^i r_{t+i} +\cdots|ha)\\
    & = {\sum_{z \in \Sigma^*}\sum_{o\in\cO}\gamma^{|z|}\tilde{R}(haoz)\mathbb{P}^{\Pi}(haoz|ha)}\\
    & = \frac{\sum_{o\in\cO}\sum_{z \in \Sigma^*}\gamma^{|z|}\tilde{R}(haoz)\mathbb{P}^{\Pi}(haoz)}{\mathbb{P}^{\Pi}(ha)}\\
    & = \frac{\sum_{o\in\cO}\sum_{z \in \Sigma^*}\gamma^{|z|}\tilde{R}(haoz)\mathbb{P}^{\Pi}(haoz)/\Pi(a|h)}{\mathbb{P}^{\Pi}(h)}\\
    &:= \frac{\tilde{Q}^{\Pi}(h, a)}{\Prob^{\Pi}(h)}
\end{align*}
where we will refer to the function $\tilde{Q}^\Pi(h, a)$ as the \emph{\longmethod{}}(\method{}). 
It is trivial to show that given the same trajectory $h$:
$$\tilde{Q}^\Pi(h, \cdot) \propto Q^\Pi(h, \cdot)$$
Therefore, we have $\argmax_{a\in\mathcal{A}}{Q}^\Pi(h, a) = \argmax_{a\in\mathcal{A}}\tilde{Q}^\Pi(h, a)$ and we can then plan according to the \method{} instead of $Q^\Pi$. 

\subsection{A spectral learning algorithm for \method{}}
In this section, we will present our spectral learning algorithm for \method{}. First, we will show that the value of a UQF given a past trajectory can be computed via a WFA. Let us denote $\sum_{z \in \Sigma^*}\gamma^{|z|}\tilde{R}(hz)\mathbb{P}^{\Pi}(hz)$ by $\tilde{V}^\Pi(h)$, we have:
$$\tilde{Q}^\Pi(h, a) = \frac{\sum_{o \in \cO}\tilde{V}^\Pi(hao)}{{\Pi}(a|h)}$$
Assume the probabilistic sampling policy $\Pi$ is given, then we only need to compute the value of the function $\tilde{V}^{\Pi}(hao)$. As a special case, if $\Pi$ is a random policy that uniformly select the actions, we can replace the term $\Pi(a|h)$ by $1$ without affecting the learned policy. 

It turns out that the function $\tilde{V}^\Pi$ can be computed by a WFA. To show that this is true,  we first introduce the following lemma stating that the function $\tilde{R}(h)\mathbb{P}^\Pi(h)$ can be computed by a WFA $B$:
\begin{lemma}
\label{rp}
Given a POMDP $\psi = \langle \ten{T},\ten{O}, \mat{r}, \cA,\cO, S,\vec{\mu},\gamma\rangle$ of size k and a sampling policy $\Pi$ induced by $\mat{\Pi}\in [0,1]^{\cS\times \cA}$, there exists a WFA $B = \langle \vec{\beta}, \{\mat{B}_{\sigma}\}_{\sigma\in \Sigma}, \vec{\tau} \rangle$ with $k$ states that realizes the function $g(h) = \tilde{R}(h)\mathbb{P}^\Pi(h)$, where $\Sigma = \cA \times \cO$ and $h \in \Sigma^*$.
\end{lemma}
\begin{proof}
Let $s^i$ denote the $i^\text{th}$ state and let $$\Tilde{\mat{O}}_{ao} = \mathrm{diag}(\ten{O}_{s^1, a, o}, \ten{O}_{s^2, a, o}, \cdots, \ten{O}_{s^k, a, o}),$$
$$\Tilde{\mat{M}}_a = \mathrm{diag}(\bm{\Pi}_{s^1, a},\bm{\Pi}_{s^2, a}, \cdots, \bm{\Pi}_{s^k, a}).$$
We can construct a WFA $B = \langle \vec{\beta}^\top, \{\mat{B}_{\sigma}\}_{\sigma\in \Sigma}, \vec{\tau} \rangle$ such that: $\vec{\beta}^\top = \vec{\mu}^\top$, $\mat{B}_{\sigma} = \mat{B}_{ao} = \Tilde{\mat{M}}_{a}\ten{T}_{:, a, :}\Tilde{\mat{O}}_{ao}$, $\bm{\tau} =  \mat{r}$.
Then by construction, one can check that the WFA $B$ computes the function $g$, which also shows that the rank of the function $g$ is at most $k$. 
\end{proof}
In fact, we can show that the function $\tilde{V}^\Pi$ can be computed by another WFA $A$, and one can easily convert $B$ to $A$. 
\begin{theorem}
\label{vtilde}
Given a POMDP $\psi$ of size k, a sampling policy $\Pi$ and a WFA $B = \langle \vec{\beta}^\top, \{\mat{B}_{\sigma}\}_{\sigma\in \Sigma}, \vec{\tau} \rangle$ realizing the function $g:h \mapsto \tilde{R}(h)\mathbb{P}^\Pi(h)$ \emph{such that the spectral radius $\rho(\gamma\sum_{\sigma\in\Sigma}\mat{B}_{\sigma})<1$}, the WFA $A = \langle \vec{\beta}^\top, \{\mat{B}_{\sigma}\}_{\sigma\in \Sigma}, (\mat{I} - \gamma\sum_{\sigma \in \Sigma}\mat{B}_{\sigma})^{-1}\vec{\tau} \rangle$  of size $k$ realizes the function $\tilde{V}^\Pi(h) = \sum_{z \in \Sigma^*}\gamma^{|z|}\tilde{R}(hz)\mathbb{P}^{\Pi}(hz)$.
\end{theorem}
\begin{proof}
By definition of the function $\tilde{V}^\Pi$, we have:
\begin{align*}
    \tilde{V}^\Pi(h) &= \sum_{z \in \Sigma^*}\gamma^{|z|}\tilde{R}(hz)\mathbb{P}^{\Pi}(hz)\\
    & = \sum_{z \in \Sigma^*}\gamma^{|z|}\vec{\beta}^\top \mat{B}_h \mat{B}_z\vec{\tau}\\
    & =\vec{\beta}^\top \mat{B}_h (\sum_{z \in \Sigma^*}\gamma^{|z|}\mat{B}_z)\vec{\tau}\\
    & = \vec{\beta}^\top \mat{B}_h (\sum_{i = 0}^\infty(\gamma \sum_{\sigma\in\Sigma}\mat{B}_{\sigma})^{i})\vec{\tau}\\
    & = \vec{\beta}^\top \mat{B}_h(\mat{I} - \gamma \sum_{\sigma \in \Sigma}\mat{B}_\sigma)^{-1}\vec{\tau}
\end{align*}
Here we applied Neumann identity: $\sum_{i = 0}^{\infty} \mat{T}^i = (\mat{I} - \mat{T})^{-1}$, which holds when $\rho(\mat{T})<1$. Therefore, the WFA $A = \langle \vec{\beta}^\top, \{\mat{B}_{\sigma}\}_{\sigma\in \Sigma}, (\mat{I} - \gamma\sum_{\sigma \in \Sigma}\mat{B}_{\sigma})^{-1}\vec{\tau} \rangle$ realizes the function $\tilde{V}^\Pi$. 
\end{proof}
Therefore, in order to compute the function $\tilde{Q}^\Pi$, we only need to learn a WFA that computes the function $g$. Following the classical spectral learning algorithm, we present our learning algorithm of POMDP planning in Algorithm~\ref{algo: qwfa}. In fact, it has been shown that the spectral learning algorithm of WFAs is statistically consistent~\cite{balle2014spectral}. Therefore our approximation of the function $\tilde{Q}^\Pi$ is consistent with respect to sample sizes. 
\RestyleAlgo{ruled}
\begin{algorithm}[h!]
\label{algo: qwfa}
\SetAlgoLined
\SetKwInOut{Input}{input}
\Input{A set of actions $\cA$, a set of observations $\cO$, discount factor $\gamma$, a probabilistic sampling policy $\Pi$, training trajectories ${D}$ and their immediate reward $\vec{y}$, rank of the truncated SVD $k$}
\KwResult{A new deterministic policy function $\pi^{new}: \Sigma^* \rightarrow \cA$}
\begin{enumerate}
    \item For a prefix $u$ and a suffix $v$, we estimate its value in the Hankel matrix as $\hat{\mat{H}}_{u, v} = \sum_{i = 0}^{|D|} \mathbb{I}_{uv}(D_i)\vec{y}_i/|D|$, where $|D|$ is the cardinality of the training set $D$, $\Sigma = \cA \times \cO$.
    \item Perform truncated SVD of rank $k$ on the estimated Hankel matrix:
    $\hat{\mat{H}} \simeq \mat{UDV}^\top$
    \item Recover the WFA $B = \langle \vec{\beta}^\top, \{\mat{B}_{\sigma}\}_{\sigma\in \Sigma}, \vec{\tau}\rangle$ realizing the function $g(h) = \tilde{R}(h)\mathbb{P}(h)$:
    $\vec{\beta}^\top = (\mat{U\Lambda})_{\lambda, :}$, $\vec{\tau} = \mat{V}^{\top}_{:,\lambda}$, $\mat{B}_{\sigma} = (\mat{UD})^+\hat{\mat{H}}_{\sigma}\mat{V}^{\top}$
    \item Convert the WFA $B$ to $A= \langle \vec{\alpha}^\top, \{\mat{A}_{\sigma}\}_{\sigma\in \Sigma}, \vec{\omega}\rangle$, which realizes the function $\tilde{V}^\Pi$, following Theorem~\ref{vtilde}, we have $\vec{\alpha}^\top = \vec{\beta}^\top$, $\mat{A}_{\sigma} = \mat{B}_{\sigma}$, and $\vec{\omega} = (\mat{I} - \gamma \sum_{\sigma\in\Sigma}\mat{B}_\sigma)^{-1}\vec{\tau}$.
    \item \textbf{Return} A new deterministic policy function $\pi^{new}$, such that given $h\in \Sigma^*$, $\pi^{new}(h) = \argmax_{a \in \mathcal{A}}\frac{\sum_{o \in \cO}\vec{\alpha}^\top\mat{A}_h\mat{A}_{ao}\vec{\omega}}{\Pi(a|h)}$
\end{enumerate}
 \caption{Spectral algorithm for UQF}
\end{algorithm}
\subsection{Scalable learning of UQF}
Now we have established the spectral learning algorithm for UQF. However, similar to the spectral learning algorithm for TPSRs, one can immediately observe 
that both time  and storage complexity are the bottleneck of this algorithm. For complex domains,
in order to obtain a complete sub-block of the Hankel matrix, one will need large amount of prefixes and suffixes to form a basis and the classical spectral learning will become intractable.

By projecting matrices down to low-dimensional spaces via randomly generated bases, \emph{matrix compressed sensing} has been widely applied in matrix compression field. In fact, previous work have successfully applied matrix sensing techniques to TPSRs~\cite{hamilton2013modelling} and developed an efficient online algorithm for learning TPSRs~\cite{hamilton2014efficient}. Here, we adopt a similar approach. 

Assume that we are given a set of prefixes $\cU$ and suffixes $\cV$ and two independent random full-rank Johnson-Lindenstrauss (JL) projection matrices $\bm{\Phi}_\cU \in \mathbb{R}^{\cU\times d_\cU}$, and $\bm{\Phi}_\cV \in \mathbb{R}^{\cV\times d_\cV}$, where $d_\cU$ and $d_\cV$ are the projection dimension for the prefixes and suffixes. 
In this work, we use Gaussian projection matrices for $\bm{\Phi}_\cU$ and $\bm{\Phi}_\cV$, which contain i.i.d. entries from the distribution $\cN(0, 1/d_\cU)$ and $\cN(0, 1/d_\cV)$, respectively.

Let us now define two injective functions over prefixes and suffixes: $\phi_\cU: \cU \to \mathbb{R}^{d_\cU}$ and $\phi_{\cV}: \cV \to \mathbb{R}^{d_\cV}$, where for all $u\in\cU$ and $v\in\cV$, we have $\phi(u) = \bm{\Phi}_{u,:}$ and  $\phi(v) = \bm{\Phi}_{v,:}$. The core step of our algorithm is to obtain the compressed estimation of the Hankel matrix, denoted by $\hat{\mat{C}}_{\cU, \cV}$ associated with the function $\tilde{R}(h)\mathbb{P}(h)$ for all $h\in \Sigma^*$. Formally, we can obtain $\hat{\mat{C}}_{\cU, \cV}$ by:
\begin{alignat*}{2}
    \hat{\mat{C}}_{\cU, \cV} &= \bm{\Phi}_{\cU}^\top\mat{H}\bm{\Phi}_{\cV}\\
    &= \sum_{i =0}^{|D|}\sum_{u, v \in \cU\times \cV} \mathbb{I}_{uv}(D_i)\vec{y}_i(\phi_U(u)\otimes\phi_V(v))
\end{alignat*}
where $D$ is the training dataset, containing all sampled trajectories, $\vec{y}$ is the vector of immediate rewards. Then, after performing the truncated SVD of $\hat{\mat{C}}_{\cU, \cV} \simeq \mat{UDV}^\top$ of rank $k$, we can compute the transition matrix for the WFA by:
\begin{figure*}
    \centering
    \vspace{-1.2\baselineskip}
    \includegraphics[width=0.8\linewidth]{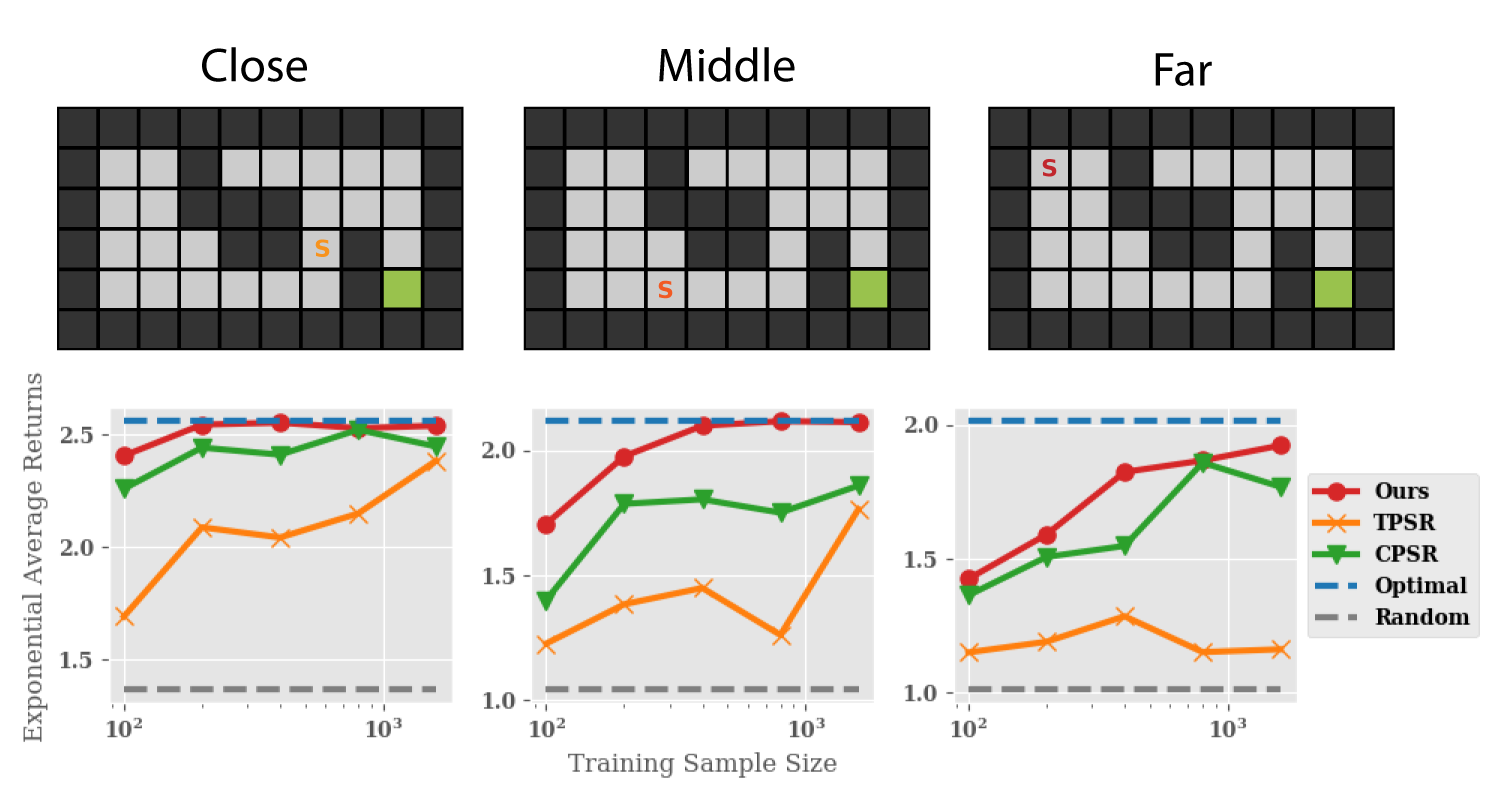}
    \caption{Experiments on three grid world tasks. The plots show the accumulated discounted rewards (returns) over 1,000 test episodes of length 100. The discount factor for computing returns is set to 0.99}
    \label{fig:gridworld_env}
    \vspace{-0.5\baselineskip}
\end{figure*}
\begin{alignat*}{2}
    \mat{B}_{\sigma}  &= (\mat{UD})^+\hat{\mat{C}}_{\cU \sigma \cV}\mat{V}\\
    &= \sum_{i =0}^{|D|}\sum_{u, v \in \cU\times \cV}\mathbb{I}_{u\sigma v}(D_i)\vec{y}_i[(\mat{UD})^+\phi(u)\otimes \mat{V}^\top\phi(v)]
\end{alignat*}
We present the complete method in Algorithm~\ref{algo-ccfq}. Instead of iterative sweeping through dataset like most planning methods do, one can build an UQF in just two passes of data: one for building the 
compressed Hankel, one for recovering the parameters. More precisely, let $L$ denote the maximum length of a trajectory in the dataset $D$, then the time complexity of our algorithm is $O(L|D|)$~\cite{hamilton2014efficient}, and there is no extra planning time needed. In contrast, fitted-Q algorithm alone requires $O(TL|D|log(L|D|))$ only for the planning stage, where $T$ is the expected number of the fitted-Q iterations. Therefore, in terms of time complexity, our algorithm is linear to the number of trajectories, leading to a very efficient algorithm.
\RestyleAlgo{ruled}
\begin{algorithm}[h!]
\label{algo-ccfq}
\SetAlgoLined
\SetKwInOut{Input}{input}
\Input{A set of actions $\cA$, a set of observations $\cO$, discount factor $\gamma$, a probabilistic sampling policy $\Pi$, training trajectories ${D}$ and their immediate reward $\vec{y}\in\mathbb{R}^{|D|}$, the rank of the truncated SVD $k$, a set of prefixes  $\cU$, a set of suffixes $\cV$, mapping functions for both prefixes and suffixes, $\phi_U$, $\phi_V$, and the corresponding projection matrix $\mat{\Phi}_{\cU}$}
\KwResult{A new deterministic policy function $\pi^{new}: \Sigma^* \rightarrow \cA$}
\begin{enumerate}
    \item Compute the compressed estimation of Hankel matrices:
    \vspace*{-0.3cm}
    \begin{center}
   $\hat{\mat{c}}_\cU = \sum_{i =0}^{|D|}\sum_{u \in \cU} \mathbb{I}_v(D_i)\vec{y}_i\phi_U(u)$\\ \vspace*{0.3cm}
    $\hat{\mat{C}}_{\cU\cV} = \sum_{i =0}^{|D|}\sum_{u, v \in \cU\times \cV} \mathbb{I}_{uv}(D_i)\vec{y}_i(\phi_U(u)\otimes\phi_V(v))$
    \end{center}
    \item Perform truncated SVD on the estimated Hankel matrix with rank $k$:
    $\hat{\mat{C}}_{\cU\cV} \simeq \mat{UDV}^\top$
    \item Recover the WFA $B = \langle \vec{\beta}^\top, \{\mat{B}_{\sigma}\}_{\sigma\in \Sigma}, \vec{\tau}\rangle$ realizing the function $g(h) = \tilde{R}(h)\mathbb{P}(h)$:
        $$\vec{\beta}^\top = \vec{e}^\top\mat{UD}, \hspace*{1.5cm}
        \vec{\tau} = \mat{D}^{-1}\mat{U}^\top\hat{\mat{c}}_{\cU}$$
        \vspace{-1.5\baselineskip}
        $$\mat{B}_{\sigma}  = \sum_{i =0}^{|D|}\sum_{u, v \in \cU\times \cV}\mathbb{I}_{u\sigma v}(D_i)\vec{y}_i[\mat{D}^{-1}\mat{U}^\top\phi(u)\otimes \mat{V}^\top\phi(v)]$$
        where $\vec{e}$ is a vector s.t. $\vec{e}^\top\bm{\Phi}_{\cU}^\top = (1,0,\cdots,0)^\top$
    \item Following Theorem~\ref{vtilde}, convert the WFA $B$ to $A= \langle \vec{\alpha}^\top, \{\mat{A}_{\sigma}\}_{\sigma\in \Sigma}, \vec{\omega}\rangle$.
    \item \textbf{Return} A new deterministic policy function $\pi^{new}$ defined by  
    $\pi^{new}(h) = \argmax_{a \in \mathcal{A}}\frac{\sum_{o \in \cO}\vec{\alpha}^\top\mat{A}_h\mat{A}_{ao}\vec{\omega}}{\Pi(a|h)}$
\end{enumerate}
 \caption{Scalable spectral algorithm for UQF}
 \vspace{-0.5\baselineskip}
\end{algorithm}
\subsection{Policy iteration}
Policy iteration has been widely applied in both MDP and POMDP settings~\cite{bellman1957dynamic,sutton1998introduction}, and have shown benefits from both empirical and theoretical perspectives~\cite{bellman1957dynamic}. It is very natural to apply policy iteration to our algorithm, since we directly learn a policy from data. The policy iteration algorithm is listed in Algorithm~\ref{policy_iter}. Note that for re-sampling, we convert our learned deterministic policy to a probabilistic one in an $\epsilon$-greedy fashion.
\RestyleAlgo{ruled}
\begin{algorithm}[h!]
\label{policy_iter}
\SetAlgoLined
\SetKwInOut{Input}{input}
\Input{An initial deterministic policy $\pi$, $\epsilon$-greedy factor $\epsilon$, a decay rate for $\epsilon$-greedy $\eta>1$, number of policy iterations $n$, number of trajectories $N$}
\KwResult{The final policy function $\pi^{fin}: \Sigma^* \rightarrow \cA$}
\begin{enumerate}
    \item Convert the deterministic policy $\pi$ to a probabilistic policy $\Pi$ in an $\epsilon$-greedy fashion: at each step, with probability $1-\epsilon$ select the optimal action according to $\pi$, with probability $\epsilon$ select a random action.\footnotemark
    \item Sample $N$ trajectories based on policy $\Pi$ 
    \item Execute Algorithm~\ref{algo: qwfa} or Algorithm~\ref{algo-ccfq} and obtain the corresponding new policy $\pi^{new}$.
    \item $\pi^{new} \to \pi $, $\epsilon/\eta \to \epsilon$
    \item Repeat all the above for $n$ times.

    \item \textbf{Return} The final policy function $\pi^{fin} = \pi^{new}$
\end{enumerate}
\caption{Policy iteration for UQF}
\end{algorithm}
\footnotetext{Note for the first iteration, one can set $\epsilon = 1$, resulting in a pure random sampling policy.}

\section{Experiments}
To assess the performance of our method, we conducted experiments on two domains: a toy grid world environment and the S-Pocman game~\cite{hamilton2014efficient}. We use TPSR/CPSR + fitted-Q as our baseline method. Our experiments have shown that indeed, in terms of sample complexity and time complexity, we outperm the classical two-stage algorithms. 
\subsection{Grid world experiment}
The first benchmark for our method is the simple grid world environment shown in Fig.~\ref{fig:gridworld_env}. The agent starts in the tile labeled $S$ and must reach the green goal state. At each time step, the agent can only perceive the number of surrounding walls and proceeds to execute one of the four actions: go up, down, left or right. 
To make the environment stochastic, with probability 0.2, the execution of the action will fail, resulting instead in a random action at the current time step. The reward function in this navigation task is sparse: the agent receives no reward until it reaches the termination state. We ran three variants of the aforementioned grid world, each corresponding to a different starting state. As one can imagine, the further away the goal state is from the starting state, the harder the task becomes.

We used a random policy to generate training data, which consisted of trajectories of length up to 100. To evaluate the policy learned by the different algorithms, we let the agent execute the learned policy for 1,000 episodes and computed the average accumulated discounted rewards, with discount factor being 0.99. The maximum length for test episodes was also set to 100. Hyperparameters were selected using cross-validation (i.e. the number of rows and columns in the Hankel matrices, the rank for SVD and $\gamma$). 
\begin{figure}[t]
    \vspace{-0.5\baselineskip}
    \centering
    \includegraphics[width=0.7\linewidth]{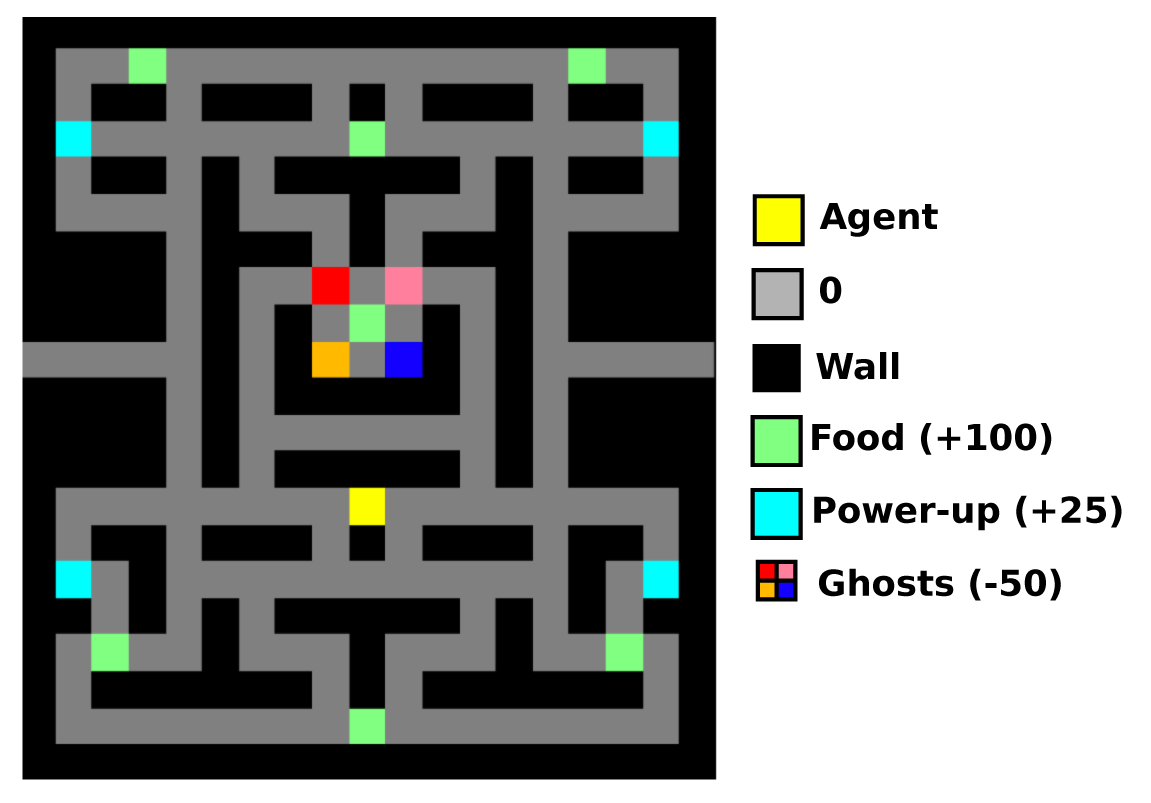}
    \caption{S-PocMan domain}
    \label{fig:my_label}
    \vspace{-1.3\baselineskip}
\end{figure}
As a baseline, we use the classical TPSRs and CPSRs as the learning method for the environment, and fitted-Q algorithm as the planning algorithm. We also performed a hyperparameter search for the baseline methods using cross validation. In addition, we report the rewards collected by a random policy as well as the optimal policy for comparison.

Results on this toy domain (see Figure \ref{fig:gridworld_env}) highlight the sample and time efficiency achieved by our method. Indeed, our algorithm outperforms the classical CPSR+fitted-Q method in all three domains, notably achieving better performance in small data regime, showing significant sample efficiency. Furthermore, it is clear that our algorithm reaches consistently to the optimal policy as sample size increases. In addition, our methods are much faster than other compared methods. For example, for the  experiment with 800 samples, to achieve similar results, our method is approximately $100$ times faster compared to CPSR+fitted-Q.

\subsection{S-PocMan domain}
For the second experiment, we show the results on the S-PocMan environment~\cite{hamilton2014efficient}. The partially observable version of the classical game Pacman was first introduced by Silver and Veness~\cite{silver2010monte} and is referred to as PocMan. In this domain, the agent needs to navigate through a grid world to collect food and avoid being captured by the ghosts. It is an extremely large partially observable domain with $10^{56}$ states~\cite{veness2011monte}. However, Hamilton et al. showed that if one were to treat the partially observable environment as if it was fully observable, a simple memoryless controller can perform extremely well under this set-up,  due to extensive reward information~\cite{hamilton2014efficient}. Hence, they proposed a harder version of PocMan, called S-PocMan. In this new domain, they drop the parts of the observation vector that allow the agent to sense the direction of the food and greatly sparsify the amount of food in the grid world, therefore making the environment more partially observable. 
\begin{table}[t]
\vspace{-1.2\baselineskip}
\caption{Training time for one policy iteration and averaged accumulated discounted rewards on S-PocMan trained on 500 trajectories.}
\vspace{0.5\baselineskip}
\label{pocman}
\centering 
\begin{tabular}{l c c c} 
\hline\hline 
 Method & \shortstack{Fitted-Q\\ Iterations} & Time (s) & Returns 
\\ [0.5ex]
\toprule 
UQF & - & \textbf{2} & \textbf{-92}  \\[1ex]
 & 400 & 489 & -101  \\[-1ex]
\raisebox{1.5ex}{CPSR} 
& 100 & 116 & -109  \\
& 50 & 60 & -150\\
& 10 & 15 & -200
\\[1ex]
\hline 
\end{tabular}
\label{tab:runtime_pocman}
\vspace{-1.2\baselineskip}
\end{table}%
In this experiment, we only used the
combination CPSR+fitted Q for our baseline algorithm, as TPSR can not scale to the large size of this environment. 
Similarly to the grid world experiment, we select the best hyperparameters through cross validation. The discount factor for computing returns was set to be 0.99999 in all runs. Table~\ref{pocman} shows the run-time and average return for both our algorithm and the baseline method. 
One can see that UQF achieves better performance compared to  CPSR+fitted-Q. Moreover, UQF exhibits significant reduction in running time: about 200 times faster than CPSR+fitted-Q. Note that building CPSR takes similar amount of time to our method, however, the extra iterative fitted-Q planning algorithm takes considerably more time to converge, as our analysis showed in section 3.3. 


\section{Conclusion}
In this paper, we propose a novel learning and planning algorithm for partially observable environments. The main idea of our algorithm relies on the estimation of the \longmethod{} with the spectral learning algorithm. Theoretically, we show that in POMDP, \method{} can be computed via a WFA and consequently can be provably learned from data using the spectral learning algorithm for WFAs. Moreover, \method{} combines the learning and planning phases of reinforcement learning together, and learns the corresponding policy in one step. Therefore, our method is more sample efficient and time efficient compared to traditional POMDP planning algorithms. This is further shown in the experiments on the grid world and S-PocMan environments.

Future work include exploring some theoretic properties of this planning approach. For example, a first step would be to obtain  convergence guarantees for policy iteration based on the \method{} spectral learning algorithm. 
In addition, our approach could be extended to the multitask setting by  leveraging the multi-task learning framework for WFAs proposed in~\cite{rabusseau2017multitask}.  
Readily, since we combine the environment dynamics and reward information together, our approach should be able to deal with partially shared environment and reward structure, leading to a potentially flexible multi-task RL framework.
\subsection*{Acknowledgement}
This research is supported by the Canadian Institute for Advanced Research (CIFAR AI chair program) and Fonds de Recherche du Québec – Nature et technologies (no. 271273). We also thank Compute Canada and Calcul Québec for the computing resources.

{
\bibliographystyle{plain}
\bibliography{refs.bib}

\begin{thebibliography}{10}

\bibitem{achlioptas2003database}
Dimitris Achlioptas.
\newblock Database-friendly random projections: Johnson-lindenstrauss with
  binary coins.
\newblock {\em Journal of computer and System Sciences}, 66(4):671--687, 2003.

\bibitem{balle2014spectral}
Borja Balle, Xavier Carreras, Franco~M Luque, and Ariadna Quattoni.
\newblock Spectral learning of weighted automata.
\newblock {\em Machine learning}, 96(1-2):33--63, 2014.

\bibitem{balle2014methods}
Borja Balle, William Hamilton, and Joelle Pineau.
\newblock Methods of moments for learning stochastic languages: Unified
  presentation and empirical comparison.
\newblock In {\em International Conference on Machine Learning}, pages
  1386--1394, 2014.

\bibitem{baraniuk2009random}
Richard~G Baraniuk and Michael~B Wakin.
\newblock Random projections of smooth manifolds.
\newblock {\em Foundations of computational mathematics}, 9(1):51--77, 2009.

\bibitem{bellman1957dynamic}
Richard Bellman.
\newblock Dynamic programming.
\newblock {\em Princeton University Press}, 1957.

\bibitem{boots2011closing}
Byron Boots, Sajid~M Siddiqi, and Geoffrey~J Gordon.
\newblock Closing the learning-planning loop with predictive state
  representations.
\newblock {\em The International Journal of Robotics Research}, 30(7):954--966,
  2011.

\bibitem{carlyle1971realizations}
Jack~W. Carlyle and Azaria Paz.
\newblock Realizations by stochastic finite automata.
\newblock {\em Journal of Computer and System Sciences}, 5(1):26--40, 1971.

\bibitem{cassandra1994acting}
Anthony~R Cassandra, Leslie~Pack Kaelbling, and Michael~L Littman.
\newblock Acting optimally in partially observable stochastic domains.
\newblock 1994.

\bibitem{droste2009handbook}
Manfred Droste, Werner Kuich, and Heiko Vogler.
\newblock {\em Handbook of weighted automata}.
\newblock Springer Science \& Business Media, 2009.

\bibitem{ernst2005tree}
Damien Ernst, Pierre Geurts, and Louis Wehenkel.
\newblock Tree-based batch mode reinforcement learning.
\newblock {\em Journal of Machine Learning Research}, 6(Apr):503--556, 2005.

\bibitem{fliess1974matrices}
Michel Fliess.
\newblock Matrices de hankel.
\newblock {\em Journal de Mathématiques Pures et Appliquées}, 53(9):197--222,
  1974.

\bibitem{hamilton2014efficient}
William Hamilton, Mahdi~Milani Fard, and Joelle Pineau.
\newblock Efficient learning and planning with compressed predictive states.
\newblock {\em The Journal of Machine Learning Research}, 15(1):3395--3439,
  2014.

\bibitem{hamilton2013modelling}
William~L Hamilton, Mahdi~Milani Fard, and Joelle Pineau.
\newblock Modelling sparse dynamical systems with compressed predictive state
  representations.
\newblock In {\em International Conference on Machine Learning}, pages
  178--186, 2013.

\bibitem{izadi2008point}
Masoumeh~T Izadi and Doina Precup.
\newblock Point-based planning for predictive state representations.
\newblock In {\em Conference of the Canadian Society for Computational Studies
  of Intelligence}, pages 126--137. Springer, 2008.

\bibitem{jaeger2000observable}
Herbert Jaeger.
\newblock Observable operator models for discrete stochastic time series.
\newblock {\em Neural Computation}, 12(6):1371--1398, 2000.

\bibitem{james2004learning}
Michael~R James and Satinder Singh.
\newblock Learning and discovery of predictive state representations in
  dynamical systems with reset.
\newblock In {\em Proceedings of the twenty-first international conference on
  Machine learning}, page~53. ACM, 2004.

\bibitem{johnson1984extensions}
William~B Johnson and Joram Lindenstrauss.
\newblock Extensions of lipschitz mappings into a hilbert space.
\newblock {\em Contemporary mathematics}, 26(189-206):1, 1984.

\bibitem{juang1991hidden}
Biing~Hwang Juang and Laurence~R Rabiner.
\newblock Hidden markov models for speech recognition.
\newblock {\em Technometrics}, 33(3):251--272, 1991.

\bibitem{kaelbling1998planning}
Leslie~Pack Kaelbling, Michael~L Littman, and Anthony~R Cassandra.
\newblock Planning and acting in partially observable stochastic domains.
\newblock {\em Artificial intelligence}, 101(1-2):99--134, 1998.

\bibitem{littman2002predictive}
Michael~L Littman and Richard~S Sutton.
\newblock Predictive representations of state.
\newblock In {\em Advances in neural information processing systems}, pages
  1555--1561, 2002.

\bibitem{pineau2003point}
Joelle Pineau, Geoff Gordon, Sebastian Thrun, et~al.
\newblock Point-based value iteration: An anytime algorithm for pomdps.
\newblock In {\em IJCAI}, volume~3, pages 1025--1032, 2003.

\bibitem{pineau2006anytime}
Joelle Pineau, Geoffrey Gordon, and Sebastian Thrun.
\newblock Anytime point-based approximations for large pomdps.
\newblock {\em Journal of Artificial Intelligence Research}, 27:335--380, 2006.

\bibitem{rabusseau2017multitask}
Guillaume Rabusseau, Borja Balle, and Joelle Pineau.
\newblock Multitask spectral learning of weighted automata.
\newblock In {\em Advances in Neural Information Processing Systems}, pages
  2588--2597, 2017.

\bibitem{rosencrantz2004learning}
Matthew Rosencrantz, Geoff Gordon, and Sebastian Thrun.
\newblock Learning low dimensional predictive representations.
\newblock In {\em Proceedings of the twenty-first international conference on
  Machine learning}, page~88. ACM, 2004.

\bibitem{shi2009hash}
Qinfeng Shi, James Petterson, Gideon Dror, John Langford, Alex Smola, and SVN
  Vishwanathan.
\newblock Hash kernels for structured data.
\newblock {\em Journal of Machine Learning Research}, 10(Nov):2615--2637, 2009.

\bibitem{silver2010monte}
David Silver and Joel Veness.
\newblock Monte-carlo planning in large pomdps.
\newblock In {\em Advances in neural information processing systems}, pages
  2164--2172, 2010.

\bibitem{singh2004predictive}
Satinder Singh, Michael~R James, and Matthew~R Rudary.
\newblock Predictive state representations: A new theory for modeling dynamical
  systems.
\newblock In {\em Proceedings of the 20th conference on Uncertainty in
  artificial intelligence}, pages 512--519. AUAI Press, 2004.

\bibitem{singh2003learning}
Satinder~P Singh, Michael~L Littman, Nicholas~K Jong, David Pardoe, and Peter
  Stone.
\newblock Learning predictive state representations.
\newblock In {\em Proceedings of the 20th International Conference on Machine
  Learning (ICML-03)}, pages 712--719, 2003.

\bibitem{sondik1978optimal}
Edward~J Sondik.
\newblock The optimal control of partially observable markov processes over the
  infinite horizon: Discounted costs.
\newblock {\em Operations research}, 26(2):282--304, 1978.

\bibitem{sutton1998introduction}
Richard~S Sutton, Andrew~G Barto, et~al.
\newblock {\em Introduction to reinforcement learning}, volume 135.
\newblock MIT press Cambridge, 1998.

\bibitem{thon2015links}
Michael Thon and Herbert Jaeger.
\newblock Links between multiplicity automata, observable operator models and
  predictive state representations: a unified learning framework.
\newblock {\em The Journal of Machine Learning Research}, 16(1):103--147, 2015.

\bibitem{veness2011monte}
Joel Veness, Kee~Siong Ng, Marcus Hutter, William Uther, and David Silver.
\newblock A monte-carlo aixi approximation.
\newblock {\em Journal of Artificial Intelligence Research}, 40:95--142, 2011.

\bibitem{weinberger2009feature}
Kilian Weinberger, Anirban Dasgupta, Josh Attenberg, John Langford, and Alex
  Smola.
\newblock Feature hashing for large scale multitask learning.
\newblock {\em arXiv preprint arXiv:0902.2206}, 2009.

\end{thebibliography}
}
\end{document}